\documentclass[a4paper]{article}
\usepackage{graphicx}
\usepackage{amsthm}
\usepackage{amsmath,amsfonts,amssymb}
\usepackage{rotating}
\usepackage{subfigure}
\usepackage{color,verbatim}
\usepackage{epstopdf}
\usepackage{algorithm}
\usepackage{algorithmic}
\usepackage{multirow}
\usepackage{dcolumn}
\usepackage{booktabs}
\usepackage{caption}
\usepackage{float}
\usepackage{titlesec}
\usepackage{enumerate}
\usepackage{geometry}
\usepackage{fancyhdr}
\usepackage[marginal]{footmisc}

\titlelabel{\thetitle. }
\usepackage[round,sort&compress]{natbib}
\renewcommand{\Re}{{\rm I}\!  {\rm R}}
\newtheorem{theorem}{Theorem}
\newtheorem{lemma}{Lemma}
\newtheorem{remark}{Remark}
\newtheorem{assumption}{Assumption}

\newcommand{\be}{\begin{equation}}
\newcommand{\ee}{\end{equation}}
\newcommand{\ba}{\begin{eqnarray}}
\newcommand{\ea}{\end{eqnarray}}
\newcommand{\bas}{\begin{eqnarray*}}
	\newcommand{\eas}{\end{eqnarray*}}

\def\S{{\cal S}}

\def\bfone{{\bf 1}}

\def\bfp{{\bf p}}

\def\bfy{{\bf y}}

\def\diag{{\rm diag}}

\def\int{{\rm int\,}}

\def\bfone{{\bf 1}}

\def\bfp{{\bf p}}

\def\bfy{{\bf y}}

\def\diag{{\rm diag}}

\def\int{{\rm int\,}}

\renewcommand{\Re}{{\rm I}\!  {\rm R}}

\title{ Ordinal Distance Metric Learning with MDS for Image Ranking }

\author{Panpan Yu \\ \textsl{School of Mathematics and Statistics}
	\\ \textsl{Beijing Institute of Technology}
	\\ \textsl{Beijing, 100081, China}
	\\ \textsl{2120151335@bit.edu.cn}\\
	\\ Qing-Na Li* 
	\\ \textsl{School of Mathematics and Statistics}
	\\ \textsl{Beijing Key Laboratory on MCAACI}
	\\ \textsl{Beijing Institute of Technology}
	\\ \textsl{Beijing, 100081, China}
	\\ \textsl{qnl@bit.edu.cn}
}
	
\date{}
\begin{document}
\maketitle
\footnote{*Corresponding author}	
\captionsetup[figure]{name={Fig.},labelsep=period}
\begin{abstract}
\noindent 
Image ranking is to rank images based on some known ranked images.
In this paper, we propose an improved linear ordinal distance metric learning approach based on the linear distance metric learning model in \citet{Li2015Ordinal}. By decomposing the distance metric $A$ as $L^TL$, the problem can be cast as looking for a linear map between two sets of points in different spaces, meanwhile maintaining some data structures. The ordinal relation of the labels can be maintained via classical multidimensional scaling, a popular tool for dimension reduction in statistics. A least squares fitting term is then introduced to the cost function, which can also maintain the local data structure. The resulting model is an unconstrained problem, and can better fit the data structure. Extensive numerical results demonstrate the improvement of the new approach over the linear distance metric learning model both in speed and ranking performance.

\vspace{1em}
\end{abstract}

\section{Introduction}
Given a labeled image dataset (referred as the training set), image ranking is to find the most relevant images for a query image based on the training set. Different from binary classification and multi-classification, the labels of {the} training set in image ranking often have an order, for example, age. The two important and challenging aims for image ranking are as follows.
The first aim is to find which class the query image belongs to, and the second is to find the most relevant images in the specific class. The first aim actually falls into ordinal regression in statistics, where different approaches have been proposed, see \citet{10.1109/TKDE.2015.2457911} for a survey on ordinal regression and \citet{Qiao2015Noncrossing}, \citet{PMID:28371790} for the recent development.
However, the second aim makes image ranking different from ordinal regression since the training images having the same label with query image need to be further ranked. Therefore, {a} direct extension of methods for ordinal regression is not appropriate for image ranking.

As for the second aim, to find the most relevant images, a natural way is to use Euclidean distance between images to measure their dissimilarities. However, as we will show later, in most cases, Euclidean distance is not appropriate for dissimilarity. A practical way is to learn a distance metric (denoted as $A$) to measure the distances between images.
This is referred as distance metric learning (DML). Then for a query image, the most relevant images are those with smallest distances under metric $A$. Many DML methods have been developed for image classification and clustering tasks. For example, the SDP approach proposed by \citet{Xing2003Distance}, an online learning algorithm proposed by \citet{Shalev2004Online}, a neighborhood component analysis (NCA) by \citet{NIPS2004_2566}, and so on (\citet{BarHillel2003Learning,Shen2010Scalable,Yang2012Bayesian}). However, most of these methods didn't assume the labels are ordered. Therefore, they can not be directly used for image ranking.

Recently, \citet{Li2015Ordinal} firstly introduced ordinal DML for image ranking. By a carefully designed weighting factor based on ordinal labels, the ordinal relationship of the images is expected to be maintained. An alternating iterative update was proposed to solve the resulting nonlinear convex semidefinite programming model, which is basically a projected gradient algorithm.

On the other hand, multidimensional scaling (MDS) is an important method for dimension reduction, which has been widely used in signal processing, molecular conformation, psychometrics and social measurement. We refer to some monographs and surveys for more applications (\citet{Anjos,Borg, Dattorro, Dokmanic, Liberti}). The idea of classical MDS {(cMDS)} is to embed the given objects into a low dimensional space based on a Euclidean distance matrix. Recently, there has been great progress in MDS, such as the semismooth Newton method for nearest Euclidean distance matrix problem (\citet{Qi2013A, Qi2014Computing}), the inexact smoothing Newton method for nonmetric MDS (\citet{LiQi2017}), as well as the applications of MDS in nonlinear dimension reduction (\citet{QiDing2016,Ding}), binary code learning (\citet{Dai2016Design}), and sensor network localization (\citet{QiXiuYuan2013}).

{\bf Our Contributions}
Note that the distance metric $A$ in DML is positive semidefinite. We represent $A$ as $A = L^TL$, where $L $ is a rectangular matrix. The first contribution of our work is that we look for $L$ instead of $A$, which gets around of positive semidefinite constraint on $A$. As a result, our method does not need spectral decomposition in each iteration and thus has quite low computational complexity. Moreover, if $L$ has only a few rows, the obtained $A$ is low rank. This brings new insight on distance metric. Distances between images under $A$ are basically the Euclidean distance between new points in a new space. The second contribution is that we employ cMDS to get the ideal points in the new space, whose Euclidean distances keep the ordinal relations as the labels do. In other words, cMDS is a key step to achieve the goal of maintaining the ordinal relationship of the data. The third contribution is that we propose a new ordinal DML model, which concerns ordinal relations between images and maintains local data structure. Extensive experiments are conducted on two data sets: UMIST face dataset and FG-NET aging  dataset. The results demonstrate the efficiency and improvement of the new approach over the linear DML model in \citet{Li2015Ordinal} both in speed and ranking performance.

The organization of the paper is as follows. In Section \ref{sec-model}, we give some preliminaries about DML model in \citet{Li2015Ordinal} and cMDS. In Section \ref{sec-Algorithm}, we  propose our new approach, referred as cMDS-DML approach. In Section \ref{sec-alg-detail}, we discuss the numerical algorithm to solve the resulting unconstrained problem. In Section \ref{sec-numerical} we report the numerical results to demonstrate the efficiency of the proposed model. Final conclusions are given in Section \ref{sec-conclusion}.

Notations. We use $\S^n $ to denote the space of symmetric matrices of $n\times n$, and $\S^n_+ $ to denote the space of positive semidefinite matrices of $n\times n$, and $A\succeq 0$ means $A\in\S^n_+$. We use small bold letters to indicate vectors.
\section{Preliminaries}\label{sec-model}
In this section, we give a brief review on the linear DML model in \citet{Li2015Ordinal} and then give some preliminaries on cMDS.
\subsection{Problem Statement}
Suppose $\mathcal{X}=\{(\mathbf{x}_i,r_i):\ i=1,\cdots,n\}$ is the training set, where $\mathbf x_i\in\Re^d$, $i=1,\cdots,n$, are the observed data, and $r_i\in \Re$, $i=1,\cdots,n$, are the corresponding labels which have an order. $n$ is sample number of the training set. We need the following assumptions.
\begin{assumption}\label{assumption2.1}
Suppose there are total $m$ different ordinal labels. Assume that the data in the training set are grouped as follows
\begin{eqnarray*}
	\mathbf{x}_1, \cdots, \mathbf{x}_{i_1}, &\rm{\ with\ labels\ }& r_1 = \cdots =r_{i_1} := a_1,\\
	\mathbf{x}_{i_1+1}, \cdots, \mathbf{x}_{i_2}, &\rm{ with\ labels\ }& r_{i_1+1} = \cdots =r_{i_2} := a_2,\\
	\cdots&\cdots&\cdots\\
	\mathbf{x}_{i_{m-1}+1}, \cdots, \mathbf{x}_{i_m}, &\rm{\ with\ labels\ }& r_{i_{m-1}+1} = \cdots =r_{i_m} := a_m,
\end{eqnarray*}
where $i_m = n$, and $a_1, \cdots, a_m$ are distinct ordinal labels.
\end{assumption}
\begin{assumption}\label{assumption}
Suppose  $\mathbf{x}_1,\cdots, \mathbf{x}_n$ are zero-centralized, i.e., $\sum_{i = 1}^n\mathbf{x}_i=0$.
\end{assumption}

To rank images, the distance metric learning approach uses the distance $d_A(\cdot,\cdot)$ defined by
\[
d_A(\mathbf{x}_i,\mathbf{x}_j) = \|\mathbf{x}_i-\mathbf{x}_j\|_A = \sqrt{(\mathbf{x}_i-\mathbf{x}_j)^TA(\mathbf{x}_i-\mathbf{x}_j)},
\]
where $A\in\S^{d}$ is positive semidefinite.
The goal is then to learn an appropriate $A$, such that the distances under metric $A$ between relevant images are small. Once $A$ is obtained, the most relevant images of a query image can be provided as those with smallest distances under $A$.
To this end, one expects $A$ to have two properties.
Firstly, ordinal information needs to be preserved under $A$, that is, for $\mathbf{x}_i, \mathbf{x}_j$ with $r_i\neq r_j$,  $d_A(\mathbf{x}_i,\mathbf{x}_j)$ is small when $|r_i-r_j|$ is small.
Secondly, local geometry structure of the data needs to be maintained under $A$. That is, for $\mathbf{x}_i, \mathbf{x}_j$ with $r_i= r_j$,  $d_A(\mathbf{x}_i,\mathbf{x}_j)\approx d_{I_d}(\mathbf{x}_i,\mathbf{x}_j)$, where $I_d$ is the identity matrix of size $d$. See also \citet{Li2015Ordinal}.

\subsection{Linear Distance Metric Learning for Ranking}
As mentioned in {the} introduction, most DML approaches did not assume the labels are ordered. \citet{Li2015Ordinal} firstly proposed a method named Linear Distance Metric Learning for Ranking (LDMLR), which dealt with ordinal labels. Below we briefly review the main idea of LDMLR.

To derive LDMLR, for each $\mathbf{x}_i$, we first specify $K$ nearest data points (under Euclidean distance) with the same label as its target neighbors. The LDMLR method is to learn a metric $A$ by solving the following nonlinear convex semidefinite programming problem:
\be\label{prob-ODML}
\begin{array}{ll}
\hbox{min}_{A\in \S^d} & h(A)\\
\hbox{s.t.} & A\succeq 0,
\end{array}
\ee
where
\[
h(A) =-\sum_{i,j} \omega_{ij}d_A^2(\mathbf{x}_i,\mathbf{x}_j) + \mu\sum_{\eta_{ij} = 1} (d_A^2(\mathbf{x}_i,\mathbf{x}_j)-d_{I_d}^2(\mathbf{x}_i,\mathbf{x}_j))^2 .
\]
Here $\mu>0$ is a tradeoff parameter. $\eta_{ij} \in \{0,1\}$ indicates whether $\mathbf{x}_j$ is one of $\mathbf{x}_i$'s target neighbors, i.e.,
\begin{equation}\label{eta}
	\eta_{ij}=
\left\{
\begin{array}{ll}
1,& \hbox{ if } \mathbf{x}_j \hbox{ is the target neighbor of } \mathbf{x}_i;\\
0, & \text{otherwise}.
\end{array}
\right.
%\begin{cases}
%		1, & \hbox{ if } \mathbf{x}_j \hbox{ is the target neighbor of } \mathbf{x}_i;\\
%		0, & \text{otherwise}.
%	\end{cases}
\end{equation}
And $\omega_{ij}$ is a weighting factor defined as
\begin{equation}\label{omega}
\omega_{ij}= \left\{
\begin{array}{ll}
(|r_i-r_j|+1)^p & \hbox{ if } r_i \neq r_j;\\
0&\hbox{ otherwise,}
\end{array}
\quad \textrm{where $p> 0$.}
\right.
\end{equation}
The first term of $h(A)$ can be viewed as a penalty term of the distance between two data points if they have different labels. The weighting factor $\omega_{ij}$ is used to adjust the importance of such distances. As we can see from the definition of $\omega_{ij}$, the larger $|r_i-r_j|$ is, the bigger $\omega_{ij}$ is. If $\mathbf{x}_i$ and $\mathbf{x}_j$ have the same label, we don't want to maximize their distances, so $\omega_{ij}=0 $ in this case. The second term of $h(A)$ is trying to maintain the local structure between the images with the same label. Model (\ref{prob-ODML}) is a convex model, and can be solved by state-of-art quadratic semidefinite programming packages, such as QSDP by \cite{qsdp}. In \cite{Li2015Ordinal}, the projected gradient method is applied to solve (\ref{prob-ODML}), i.e., the following update is used
\[
A_{k+1} = \Pi_{\S^d_{+}}(A_{k}-\nabla h(A_k)),
\]
where $\Pi_{\S^d_{+}}(\cdot)$ denotes the projection onto $\S^d_+$.

In LDMLR, the ordinal relation of the images is maintained by introducing a weighting factor, which is calculated based on the ordinal labels. Furthermore, the local data structure can be kept by the second term in $h(A)$.

\subsection{Classical Multidimensional Scaling (cMDS)}
The aim of cMDS is to embed data in a lower dimensional space while preserving the distances between data.
Given the coordinates of a set of points, namely $\{\mathbf{y}_1, \ldots, \mathbf{y}_n\}$ with
$\mathbf{y}_i \in \Re^s$,
it is straightforward to compute the pairwise Euclidean distances: $d_{ij} = \|\mathbf{y}_i - \mathbf{y}_j \|$, $i,j=1,\ldots, n$.
The matrix $D = (d_{ij}^2)$ is known as the (squared) Euclidean Distance Matrix (EDM) of those points.
However, the inverse problem is more interesting and important.
Suppose $D$ is given. The method of cMDS generates a set of coordinates that preserve the pairwise distances in $D$. We give a short description of cMDS below.
Let
\be\label{B-def}
J := I - \frac 1n \bfone \bfone^T \qquad \mbox{and} \qquad
B(D) : = -\frac 12 JDJ,
\ee
where $I$ is the $n\times n$ identity matrix and $\bfone$ is the (column) vector of all ones in $\Re^n$.
In literature, $J$ is known as the centralization matrix and $B$ is the double-centralized matrix of $D$ (also
the Gram matrix of $D$ because $B$ is positive semidefinite). Suppose $B$ admits the spectral decomposition:
\be \label{cMDS-a}
B(D) = [\bfp_1, \ldots, \bfp_s] \left[ \begin{array}{ccc}
	\lambda_1 & & \\
	& \ddots & \\
	&        & \lambda_s
\end{array} \right] \left[ \begin{array}{c}
\bfp_1^T \\ \vdots \\ \bfp_s^T
\end{array} \right],
\ee
where $\lambda_1, \ldots, \lambda_s$ are positive eigenvalues of $B$ (the rest are zero) and
$\bfp_1, \ldots, \bfp_s$ are the corresponding orthonormal eigenvectors. Then the following
coordinates $\bfy_1, \ldots, \bfy_n$ obtained by
\be \label{cMDS-b}
[ \bfy_1, \bfy_2, \ldots, \bfy_n] := \left[ \begin{array}{ccc}
	\sqrt{\lambda_1} & & \\
	& \ddots & \\
	&        & \sqrt{\lambda_s}
\end{array} \right] \left[ \begin{array}{c}
\bfp_1^T \\ \vdots \\ \bfp_s^T
\end{array} \right]
\ee
preserve the known distances in the sense that $\| \bfy_i - \bfy_j \| = d_{ij}$ for all $i, j=1,\ldots, n$.
This is the well known cMDS. We refer to \citet{GOWER198581}, \citet{Schoenberg1935Remarks}, \citet{Torgerson1952}, \citet{Young1938Discussion}, \citet{Borg}, and \citet{Dattorro} for detailed description and generalizations of cMDS.

\section{A New Approach for Ranking}\label{sec-Algorithm}

In this section, we will motivate our new approach and discuss some related properties of EDM.

\subsection{A New Approach}
The idea of our approach is as follows. First, by decomposing $A=L^TL$, the problem reduces to looking for a linear map $L$ from the original space $\Re^d$ to a new space, denoted as $\Re^s$. The points $L\mathbf x_i$ in the new space are referred as the embedding points corresponding to $\mathbf x_i$, $i = 1, \cdots, n$. Then we apply cMDS to get the estimations of those embedding points, denoted as $\{\mathbf y_1, \cdots, \mathbf y_n\}$. Finally, $L$ is learned based on two sets of points $\{\mathbf x_1, \cdots, \mathbf x_n\}$ and $\{\mathbf y_1, \cdots, \mathbf y_n\}$. We detail our approach in the following three steps.

{\bf Step 1. Decompose $A=L^TL$}

A natural way of learning a distance metric $A\in\S^d$ is to decompose $A$ as $A = L^TL$, where $L\in\Re^{ s\times d}$ is a rectangular matrix and $s$ is a prescribed dimension, where $s\le d$. The decomposition has been used in several references, see for example \citet{Sugiyama2007}, \citet{Weinberger2009Distance}, \citet{Xiang2008}. Learning $L$ instead of $A$ brings us some advantages.
Firstly, it allows us to get around of the positive semidefinite constraint $A\succeq 0$, resulting in an unconstrained model.
Secondly, low rank structure of $A$ can be specified by choosing $s\ll d$. Note that given a query image, it is necessary to compute distances between {the} query image and every training image. The time complexity of computing distances should be kept as low as possible. With a low rank $A$, such complexity can be reduced from $O(d^2)$ to $O(ds)$. Finally, it provides us insights on the Mahalanobis distance metric A. $L\in\Re^{s\times d}$ is basically a linear map from $\Re^d $ to $\Re^s$. The distance between $\mathbf{x}_i$ and $\mathbf{x}_j$ under metric $A$
can be reformulated as
\begin{equation}\label{distance}
	\begin{split}
		d_A(\mathbf{x}_i,\mathbf{x}_j)
		 =\sqrt{(\mathbf{x}_i-\mathbf{x}_j)^TL^TL(\mathbf{x}_i-\mathbf{x}_j)}
		 =\| L(\mathbf{x}_i-\mathbf{x}_j)\|
		:{=}d^L(\mathbf{x}_i,\mathbf{x}_j).
	\end{split}
\end{equation}
In other words, the distance between $\mathbf{x}_i$ and $\mathbf{x}_j$ under metric $A$ is essentially the Euclidean distance of new points $ L\mathbf x_i$ and $ L \mathbf x_j$ in the space $\Re^s$.

Recall that we denote the space where $\mathbf x_i$ lies in (i.e., $\Re^d$) as the original space, the space where $L\mathbf x_i$ lies in (i.e., $\Re^s$) as the new space, and $L\mathbf x_i$ is referred as the embedding point of $\mathbf x_i$.  Now image ranking reduces to looking for a linear map, which maps $\mathbf x_i$ to a proper new space such that the following properties hold.
\begin{itemize}
\item [(i)] The distances between embedding points can well reflect the corresponding ordinal labels. In other words, the Euclidean distances between embedding points with different labels should follow the order of their label differences, i.e.,
\[
 \|L\mathbf x_i-L\mathbf x_j\|> \|L\mathbf x_i-L\mathbf x_k\|, \hbox{ if } |r_i-r_j|>|r_i-r_k|, \ r_i\neq r_j, \ r_i\neq r_k, \ r_j\neq r_k.
\]
 \item [(ii)] Local data structure must be maintained. That is, the Euclidean distances between a point and its target neighbors with the same label in the original space  need to be maintained as much as possible in the new space. That is,
\[d_A(\mathbf x_i,\mathbf x_j)\approx d_{I_d}(\mathbf{x}_i,\mathbf{x}_j),\hbox{ if } r_i= r_j{\hbox{ and }\mathbf{x}_j \hbox{ is the target neighbor of } \mathbf{x}_i}.\]

\end{itemize}
In the following, we apply cMDS to get the estimations $\{\mathbf y_1, \cdots, \mathbf y_n\}$ of the embedding points in a new space, which {enjoy} property (i), then learn a linear mapping $L$ based on two sets of points $\{\mathbf x_1, \cdots, \mathbf x_n\}$ and $\{\mathbf y_1, \cdots, \mathbf y_n\}$.

{\bf Step 2. Apply cMDS}

In order to apply cMDS to get the estimations of embedding points, an EDM is needed. Note that the points with the same label can be basically viewed as one point, and further inspired by the weighting factor defined in (\ref{omega}), we can construct an EDM based on the ordinal labels. A trivial choice is to define $D$ by $D_{ij} = (|r_i-r_j|^2), \ i,j= 1, \cdots, n.$ However, from numerical point of view, we can further add a parameter $\beta$ to $|r_i-r_j|$ to allow more flexibility. This leads to the following form of $D$.
 Define $D\in \S^n$ as
\be\label{D}
	D_{ij}=
	\left\{
	\begin{array}{ll}
		(|r_i-r_j|+\beta)^2,\quad\quad& \text{if }  r_i\ne r_j;
		\\
		0, & \hbox{otherwise.}
	\end{array}
	\right.
\ee
Under Assumption \ref{assumption2.1}, let
\be\label{bardelta}
	\bar \delta_{ij}=
	\left\{
	\begin{array}{ll}
		 |a_i-a_j| ,\quad\quad& \text{if }  i\ne  j, \ i, j = 1, \cdots,m;
		\\
		0, & \hbox{otherwise.}
	\end{array}
	\right.
\ee
The following theorem shows  that if $\beta$ is properly chosen, then $D$ is an EDM.
\begin{theorem}\label{thm}
Let $\overline\Delta^{\frac12}:=(\bar \delta_{ij})$ and $\mu_0$ is the smallest eigenvalue of $-\frac12 J\overline \Delta^{\frac12}J$. If $\beta\ge-4\mu_0$, then $D$ defined by (\ref{D}) is an EDM.
		
\end{theorem}
The proof is postponed in Section \ref{sec-proof}.	
If $D$ is not an EDM, we refer to \citet{Qi2013A}, \citet{LiQi2017} for more details. By applying cMDS to $D$, we can get the estimations $\{\mathbf{y}_1, \cdots, \mathbf{y}_n\}$ of embedding points in the new space.

\begin{remark}
For $\mathbf x_i$ and $\mathbf x_j$ with $r_i=r_j$, their estimations of embedding points $\mathbf y_i, \mathbf y_j$ basically collapse to one point, since $\|\mathbf y_i-\mathbf y_j\|=D_{ij} = 0$. For $\mathbf x_i$ and $\mathbf x_j$ with $r_i\neq r_j$, the Euclidean distance between their estimations $\mathbf y_i,\ \mathbf y_j$ of embedding points is $ \|\mathbf y_i-\mathbf y_j\|=D^{\frac12}_{ij} = |r_i-r_j|+\beta$. Consequently, there is
\[
 \|\mathbf y_i-\mathbf y_j\|> \|\mathbf y_i-\mathbf y_k\|, \hbox{ if } |r_i-r_j|>|r_i-r_k|, \ r_i\neq r_j, \ r_i\neq r_k, \ r_j\neq r_k.
 \]
In other words, { $\{ \mathbf y_1, \cdots, \mathbf y_n \}$} enjoy property $(i)$.
\end{remark}

{\bf Step 3. Matching Two Sets of Points}

The final step is to learn $L$ based on two sets of points $\{\mathbf x_1, \cdots, \mathbf x_n\}$ and $\{\mathbf y_1, \cdots, \mathbf y_n\}$ to make $L$ have properties (i) and (ii). To deal with property (i), we need to match $\{\mathbf x_1, \cdots, \mathbf x_n\}$ and $\{\mathbf y_1, \cdots, \mathbf y_n\}$ as much as possible since $ \mathbf y_1, \cdots, \mathbf y_n $ already satisfy property (i). A natural statistical way is to use a least squares fitting term. To tackle property (ii), we adopt the second term of $h(A)$ in (\ref{prob-ODML}), since it does a good job based on the numerical performance. Now we reach the following model
\be\label{prob}
\hbox{min}_{L\in\Re^{s\times d},c\in \Re} \ f(L,c):=\frac{1}{2}\sum_{i=1}^n\| L\mathbf{x}_i-c\mathbf{y}_i\|^2+ \mu \sum_{\eta_{ij}=1}((d^{L } (\mathbf{x}_i,\mathbf{x}_j))^2-d_{I_d}^2(\mathbf{x}_i,\mathbf{x}_j))^2,
\ee
where $\eta_{ij}$ is defined as in (\ref{eta}). To allow more flexibility, we also use a scaling variable $c\in \Re$ in the fitting term.

Although (\ref{prob}) is a nonconvex model in $L$, the proposed approach enjoys the following good properties.
\begin{itemize}
\item By dealing with $L$ instead, the resulting model (\ref{prob}) is an unconstrained problem, which allows various numerical algorithms to solve. Further, we can emphasize the low rank structure of $A$ by restricting $L$ to be a short fat matrix, i.e., $s\ll d$.
\item By applying cMDS, we take into account of the ordinal information of labels, which leads us a good estimation of embedding points.
\item By matching $\{\mathbf x_1,\cdots, \mathbf x_n\}$ with $\{\mathbf y_1, \cdots,\mathbf  y_n\}$ with the least squares fitting term, hopefully, the resulting embedding points will also keep property (i). Our numerical results {actually} verify this observation.
\end{itemize}
\subsection{Proof of Theorem \ref{thm}}\label{sec-proof}

Define $\Delta \in\S^m$ as %	$\Delta = (\delta_{ij}^2)$ where
\begin{equation}\label{delta}
\Delta = (\delta_{ij}^2), \ \hbox{where}\ \ \delta_{ij}=
\left\{
\begin{array}{ll}
\bar\delta_{ij}+\beta ,\quad\quad& \text{if}\quad  i\ne  j, \ i, j = 1, \cdots,m;\\
0, & \hbox{otherwise.}
\end{array}
\right.
\end{equation}
Then we have the following lemma.
\begin{lemma}\label{lem1}
Let $D\in\S^n$ and $\Delta\in\S^m$ be defined as in (\ref{D}) and (\ref{delta}). Let Assumption \ref{assumption} hold.
$D$ is an EDM if and only if $\Delta$ is an EDM.
\end{lemma}	
\begin{proof}[\bf{Proof.}]
Suppose $D$ is an EDM generated by points $\{\mathbf  y_1, \cdots, \mathbf  y_n\}$. By the definition of $D$, there is \[
	\|\mathbf  y_i-\mathbf  y_j\|=0,\quad \hbox{if } r_i = r_j,
\]
which implies that $\mathbf  y_{i_{t-1}+1} = \cdots= \mathbf  y_{i_t}$, $t = 1,\cdots, m$. Let $\mathbf  y_{i_{t-1}+1} = \cdots=\mathbf   y_{i_t}:= \mathbf  z_t$, $t = 1, \cdots,m$. Obviously, $\Delta$ is an EDM generated by points $\{\mathbf  z_1,\cdots,\mathbf  z_m\}$. Conversely, suppose that $\Delta$ is an EDM generated by points $\{\mathbf  z_1,\cdots, \mathbf  z_m\}$. Let $\mathbf  y_{i_{t-1}+1} = \cdots= \mathbf  y_{i_t} =\mathbf   z_t$, $t = 1, \cdots, m$. One can show that $D$ is an EDM generated by $\{\mathbf  y_1, \cdots, \mathbf  y_n\}$. The proof is finished.
\end{proof}		
Next, we show that $\Delta$ is an EDM if $\beta$ is properly chosen. %The proof is inspired by that in \citet{Cailliez1983}.
\begin{lemma}\label{lem2}
Let $\overline\Delta^{\frac12}:=(\bar \delta_{ij})$ and $\mu_0$ is the smallest eigenvalue of $-\frac12 J\overline \Delta^{\frac12}J$. If $\beta\ge-4\mu_0$, then $\Delta$ defined by (\ref{delta}) is an EDM.			
\end{lemma}
\begin{proof}[\bf Proof.]
It is well known (\citet{Schoenberg1935Remarks, Young1938Discussion}) that $\Delta$ is an EDM if and only if
\[
\diag(\Delta) =0 \  \hbox{and} \ B(\Delta) =-\frac{1}{2}J\Delta J\succeq 0 .
\]
Also note that
\begin{equation}\label{propertyJ}
J\mathbf \mathbf  \mathbf{1} =0, \ B(\Delta)J = B(\Delta), \ JB(\Delta) = B(\Delta).
\end{equation}
To prove that $\Delta$ is an EDM, we only need to show the positive semidefiniteness of $B(\Delta)$. Let $\overline \Delta = (\bar \delta_{ij}^2)$. Note that
\begin{equation*}
	B(\Delta) = B(\overline \Delta) + 2\beta B(\overline\Delta^{\frac12}) + \frac12\beta^2 J.
\end{equation*}
It suffices to show if $\beta\ge-4\mu_0$, then for any $x\in\Re^m$, there is
\begin{equation*}
	\mathbf  x^TB(\overline \Delta)\mathbf  x + 2\beta\mathbf   x^TB(\overline\Delta^{\frac12})\mathbf  x + \frac12\beta^2\mathbf  x^TJ\mathbf x\ge0.
\end{equation*}
	
Obviously, $\overline\Delta$ is an EDM. Consequently, for any $\mathbf  x\in\Re^m$, $\mathbf x^TB(\overline\Delta)\mathbf x\ge0$. Further, $B(\overline\Delta^{\frac12})-\mu_0I\succeq0$ implies that
\[
	\mathbf x^T(B(\overline\Delta^{\frac12})-\mu_0I)\mathbf x\ge0, \ \forall\ \mathbf x\in\Re^m.
\]
By substituting $\mathbf{x}$ by $J\mathbf{x}$ and noting equalities in \eqref{propertyJ}, we have
\[
	\mathbf x^TB(\overline \Delta^{\frac12})\mathbf x-\mu_0 \mathbf x^TJ\mathbf x\ge0, \ \forall\ \mathbf x\in\Re^m.
\]
It gives that
\[
	\mathbf x^TB(\overline \Delta)\mathbf x + 2\beta\mathbf  x^TB(\overline\Delta^{\frac12})\mathbf x + \frac12\beta^2\mathbf x^TJ\mathbf x\ge \beta(2\mu_0+\frac\beta2)\mathbf x^TJ\mathbf x\ge0,
\]
where the last inequality follows by the assumption $\beta\ge-4\mu_0$ as well as the positive semidefiniteness of $J$. The proof is finished.
\end{proof}

The proof of Lemma \ref{lem2} is inspired by Theorem 1 in \citet{Cailliez1983}. The difference is that $B(\overline\Delta)$ is an EDM and $\beta$ is allowed to be negative in Lemma \ref{lem2}.

\begin{proof}[\bf Proof of Theorem \ref{thm}.] The result of Theorem \ref{thm} can be directly derived from Lemma \ref{lem1} and Lemma \ref{lem2}.
\end{proof}

\begin{remark}
Note that in cMDS, $\{\mathbf y_1,\cdots,\mathbf y_n\}$ obtained from $D$ is not unique due to the eigenvalue decomposition of $B(D)$. However, ${\mathbf y_1,\cdots,\mathbf y_n}$ are centralized, i.e., $\sum_{i = 1}^n\mathbf y_i=0$. The computational cost for generating $\{\mathbf y_1,\cdots,\mathbf y_n\}$ is $O(n^3)$. If $n$ is large, the computational cost can be further reduced to $O(m^3)$ by the following process, which is based on Lemma \ref{lem1} and Lemma \ref{lem2}. It is easy to verify that ${\mathbf y_1,\cdots,\mathbf y_n}$ generated by Algorithm \ref{alg-y} satisfy  $\sum_{i = 1}^n\mathbf y_i=0$, and the corresponding EDM is $D$ defined in (\ref{D}).
\end{remark}

\begin{algorithm}[htb]
	\caption{ Alternative way to generate $\{\bf y_1,\cdots,\bf y_n\}$ }
	\label{alg-y}
	\begin{enumerate}[\quad Step 1.]
		\item Compute $\Delta$ defined by (\ref{delta}).
		\item Apply cMDS to $\Delta$ to get $\mathbf{z}_1, \cdots, \mathbf{z}_m\in\Re^{s_1}$.
		\item Let $\tilde{\mathbf  y}_{i_{t-1}+1} = \cdots= \tilde{\mathbf  y}_{i_t} =(\mathbf   z_t,\mathbf 0)\in\Re^s$, $t = 1, \cdots, m$, where $\mathbf 0\in\Re^{s-s_1}$.
		\item Denote $\overline{\mathbf y}=\sum_{i = 1}^n\tilde{\mathbf y_i}$. Let $\mathbf y_i = \tilde{\mathbf y}_i -\overline{\mathbf y} $, $i = 1,\cdots, n$.
	\end{enumerate}
\end{algorithm}
\section{Numerical Algorithm}\label{sec-alg-detail}
Problem (\ref{prob}) is an unconstrained nonlinear problem, and can be solved by various algorithms. Here, we choose the traditional steepest descent method with the Armijo line search. The convergence result of the steepest descent method can be found in classical optimization books, e.g. \citet[P42]{Nocedal}. Algorithm \ref{alg} summarizes the details of our approach.

\begin{algorithm}[htb]
	\caption{ cMDS-DML for image ranking }
	\label{alg}
	\begin{itemize}
		\item [S0] Given {a} training set: $\mathbf{x}_1,\cdots,\mathbf{x}_n\in \Re^d$, and their corresponding labels $r_1, \cdots, r_n$.\\
		Initialize:	$L^0 =(\mathbf{e}_1,\ldots,\mathbf{e}_s)^T\in\Re^{s\times d}$, $c_0=1$.
		
		Parameters:$\quad$$\mu$, $\epsilon>0$,
		$\sigma\in (0,1)$, $\rho\in (0,1)$, $\gamma>0$,
		$k = 0$.
		
		\item [S1] %Obtain a Euclidean distance matrix $D$ defined in (\ref{D}) based on the  training set.
		Compute the Euclidean distance matrix $D$ according to \eqref{D}.
		\item [S2] Apply cMDS to get {estimations} of embedding points $\mathbf{y}_1, \cdots, \mathbf{y}_n\in\Re^s$.
		\item [S3] Search $K$ target neighbors in the original space $\Re^d$ for each training sample $\mathbf{x}_1$, $\dots$, $\mathbf{x}_n$.
		\item [S4] Compute $\nabla f(L^k,c_k)$.  If $\|\nabla f(L^k,c_k)\|\le\epsilon$, stop; otherwise, let $d^k = -\nabla f(L^k,c_k)$, go to S5.
		\item [S5] Apply the Armijo line search to determine a steplength $\alpha_k=\gamma\rho^{m_k}$, where $m_k$ is the smallest positive integer such that the following inequality holds
		\[
		f((L^{k},c_k)+\gamma\rho^m d^k)-f(L^k,c_k)\le \sigma\gamma\rho^m \nabla f(L^k,c_k)^Td^k.
		\]
		\item [S6] Let $(L^{k+1},c_{k+1}) = (L^k,c_k)+\alpha_k d^k$, $k=k+1$, go to S4.	
	\end{itemize}
\end{algorithm}

{\bf Implementations}
Let $X_{ij}:=(\mathbf{x}_i-\mathbf{x}_j)(\mathbf{x}_i-\mathbf{x}_j)^T$,
the gradient $\nabla f(L,c) $ takes the following form
\begin{equation*}
\begin{split}
\nabla_L f(L,c)&=\sum_{i=1}^n(L\mathbf{x}_i\mathbf{x}_i^T-c\mathbf{y}_i\mathbf{x}_i^T)
+ 4\mu\sum_{\eta_{ij}=1}L(X_{ij}L^TLX_{ij}-X_{ij}^2)\\
&=\sum_{i=1}^n(L\mathbf{x}_i\mathbf{x}_i^T-c\mathbf{y}_i\mathbf{x}_i^T)
+ 4\mu\sum_{\eta_{ij}=1}(\|L(\mathbf x_i-\mathbf x_j)\|^2-\| \mathbf x_i-\mathbf x_j \|^2)L(\mathbf{x}_i-\mathbf{x}_j)(\mathbf{x}_i-\mathbf{x}_j)^T,\\
\nabla_c f(L,c)&=c\sum_{i=1}^n\mathbf{y}_i^T \mathbf{y}_i-\sum_{i=1}^n\mathbf{y}_i^T L \mathbf{x}_i.
\end{split}
\end{equation*}

{\bf Computational Complexity}
We compare the computational complexity (mainly in multiplication and division) of Algorithm \ref{alg} with that of LDMLR, and the details are summarized in Table \ref{table00}, where steps with underline indicate the iterative steps.
Note that if $n$ is large, S2 can be replaced by Algorithm \ref{alg-y} and the computational complexity for S2 can be further reduced from $O(n^3) $ to $O(m^3)$.
For the iterative process S4-S6, the complexity for each iteration is $O(rnK d^2)$, where $r$ is the maximum number for the line search loop. In contrast, for LDMLR, the computational complexity in each iteration is $O(\hbox{max}(n^2d^2, nKd^3))$, which is higher than that of S3-S6 in Algorithm \ref{alg}, no matter $n>d$ or $n<d$.

\begin{table}[!ht]
	\centering
	\caption{Computational Complexity for Algorithm \ref{alg} and LDMLR.}
	\label{table00}
	\begin{tabular}{c l|r r}
		\hline
		\multicolumn{2}{c|}  {Algorithm \ref{alg}} &\multicolumn{2}{c}{LDMLR}\\\hline
		Step & Complexity & Complexity & Step\\
		\hline
		S0 &    $O(sd)$  &   $O(d^2)$ & Initialize \\
		\hline
		S1 &  $O( n^2)$ &\multirow{3}{*}{$O(dn^2+ Kn^2)$} &$K$ target\\%\hline
		S2&  $O(n^3 ) $ &&neighbor\\
		S3& $O(dn^2+ Kn^2)$&&  search\\\hline	
		\underline{S4} & $O(nsd+nk(d+sd+s^2))$ & $O(n^2d^2+nKd^3)$ & $\underline{\nabla h(A) }$
		\\\hline
		\underline{S5}& $O(r(nKsd+nKd^2))$&  \multirow{2}{*}{$O(d^3)$}&  \multirow{2}{*}{$\underline{\Pi_{S^n_+}(\cdot)}$} \\
		\underline{S6} & $O(ds)$ &&\\\hline	
	\end{tabular}
\end{table}
\section{Numerical Results} \label{sec-numerical}
In this section, we present some numerical results to verify the efficiency of the proposed model.
To evaluate the performance of the model, we employ the following popular procedure to assess the image ranking model.
For a given dataset, we divide it into the training set and the testing set. We first learn a distance metric based on the training set, then apply it to rank each image in the testing set. Denote by $\{\mathbf{m}_i\}_{i=1}^N$ the images in the testing set, here $N$ is the size of testing set. The estimated label $\hat{p}_i$ is obtained based on the distance in the new space. We employ the popular $k$-nearest neighbor regression to obtain $\hat{p}_i$, which is used in \citet{Li2015Ordinal}, \citet{Weinberger2009Distance}.
The mean absolute error
\(MAE=1/N\sum_{i=1}^{N}|\hat{p}_i-p_i|\)
is used as a measure to evaluate the performance. Here $p_1,\cdots,p_N$ are the true labels of test data $\mathbf{m}_1, \cdots, \mathbf{m}_N$.

We test the proposed method on the UMIST dataset (\citet{Graham1998}) and FG-NET dataset (\citet{Lanitis2008}). We also compare our method with the method LDMLR in \citet{Li2015Ordinal}. For each test problem, we repeat each experiment 50 times and report the average results. The algorithm is implemented in {Matlab} R2016a and is run on a computer with Intel Core 2 Duo CPU E7500 2.93GHz, RAM 2GB.

\subsection{Experiments on the UMIST image dataset}

The UMIST face dataset is a multiview dataset which consists of 575 images of 20 people, each covers a wide range of poses from profile to frontal views. Fig. $\ref{Figure3}$ shows some examples from the UMIST dataset.
\begin{figure}[!ht]\centering
	\includegraphics[width=1cm]{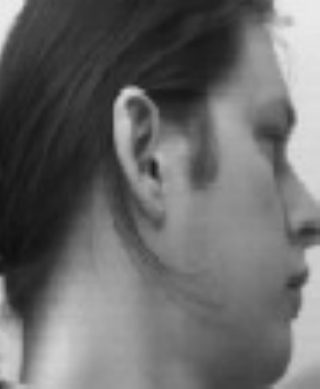}
	\includegraphics[width=1cm]{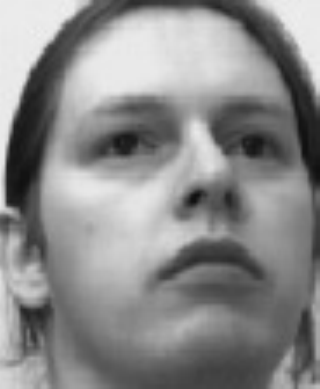}
	\includegraphics[width=1cm]{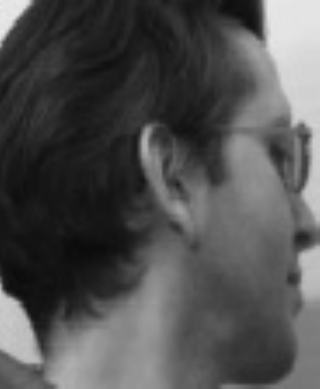}
	\includegraphics[width=1cm]{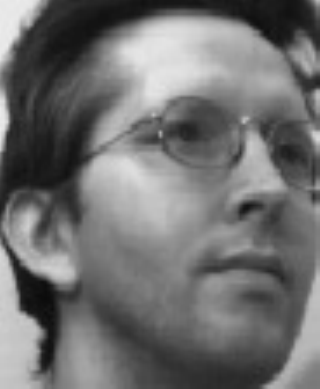}
	\includegraphics[width=1cm]{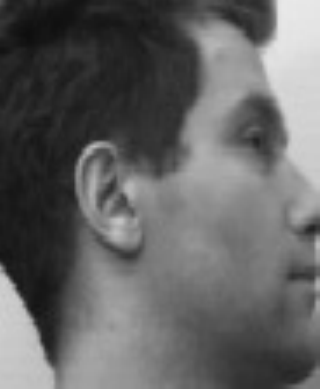}
	\includegraphics[width=1cm]{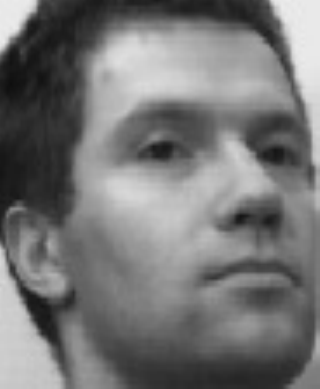}
	\includegraphics[width=1cm]{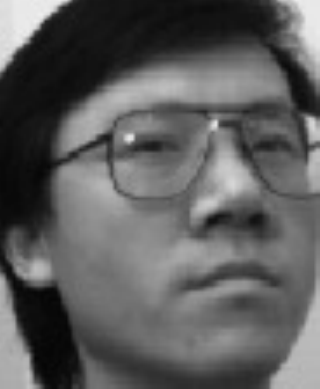}
	\includegraphics[width=1cm]{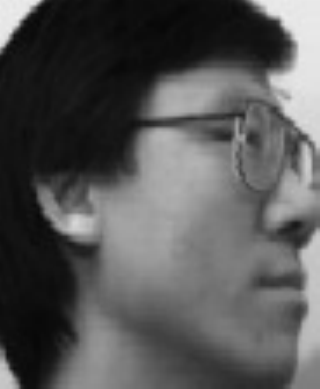}\\
	\includegraphics[width=1cm]{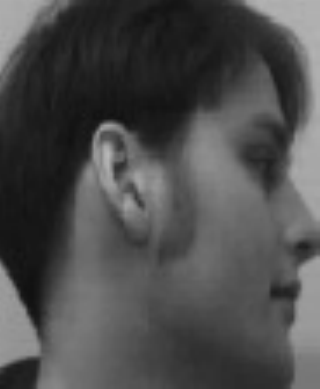}
	\includegraphics[width=1cm]{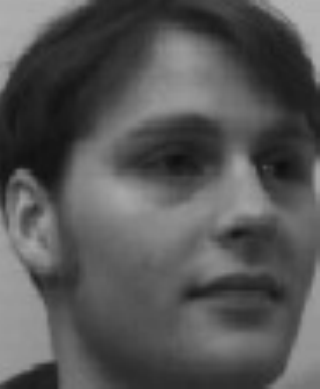}
	\includegraphics[width=1cm]{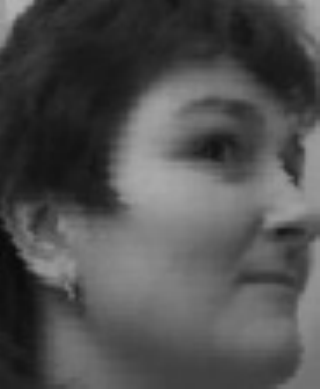}
	\includegraphics[width=1cm]{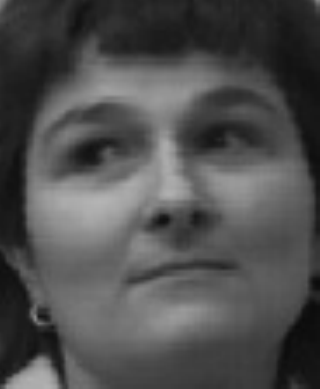}
	\includegraphics[width=1cm]{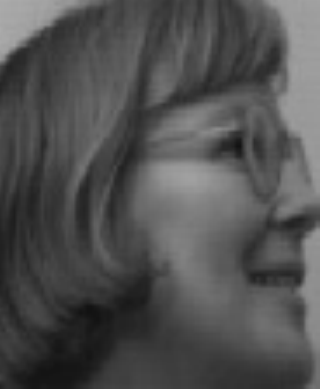}
	\includegraphics[width=1cm]{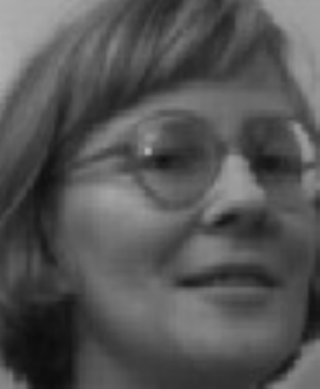}
	\includegraphics[width=1cm]{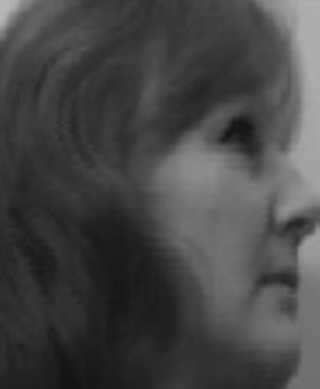}
	\includegraphics[width=1cm]{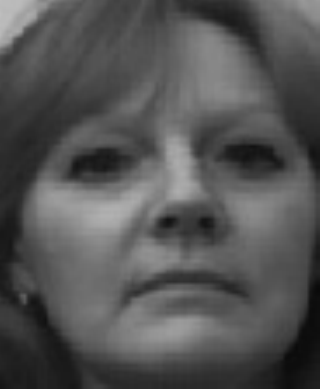}\\
	\includegraphics[width=1cm]{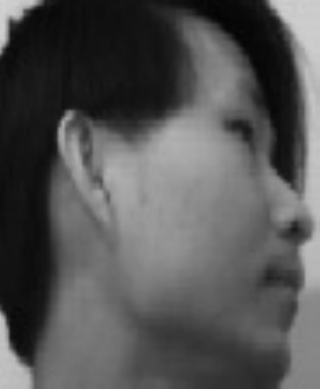}
	\includegraphics[width=1cm]{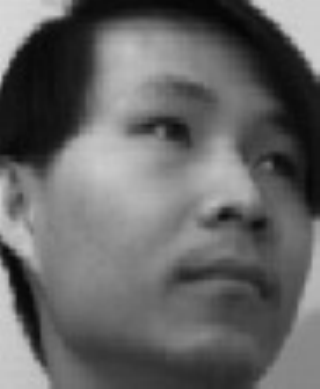}
	\includegraphics[width=1cm]{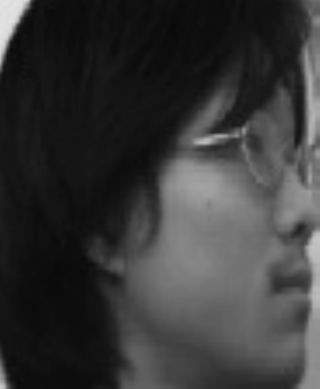}
	\includegraphics[width=1cm]{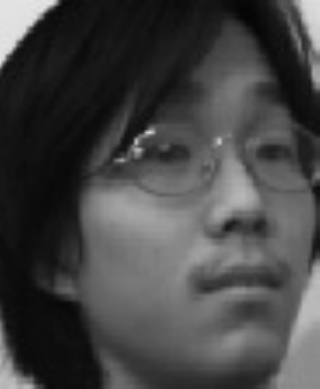}
	\includegraphics[width=1cm]{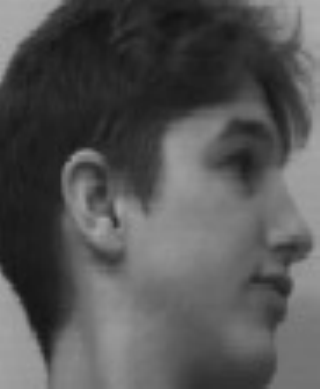}
	\includegraphics[width=1cm]{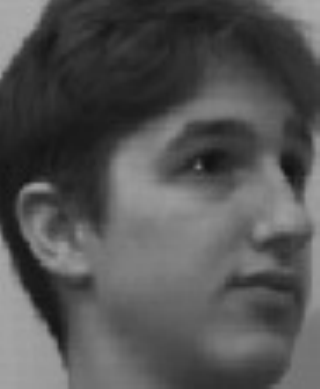}
	\includegraphics[width=1cm]{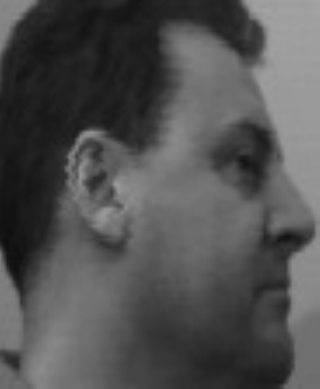}
	\includegraphics[width=1cm]{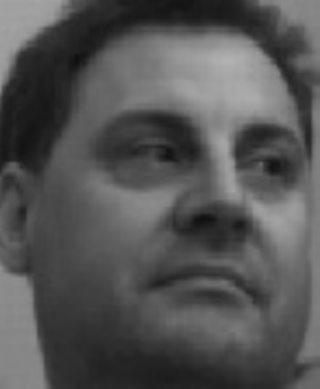}\\
	\includegraphics[width=1cm]{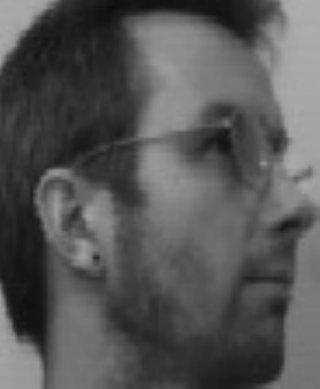}
	\includegraphics[width=1cm]{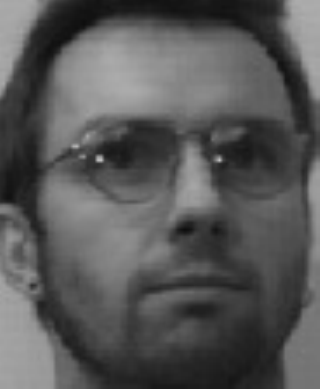}
	\includegraphics[width=1cm]{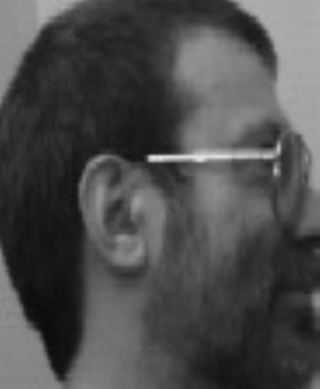}
	\includegraphics[width=1cm]{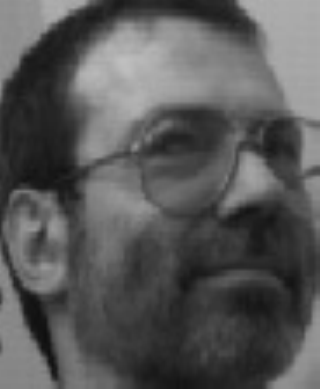}
	\includegraphics[width=1cm]{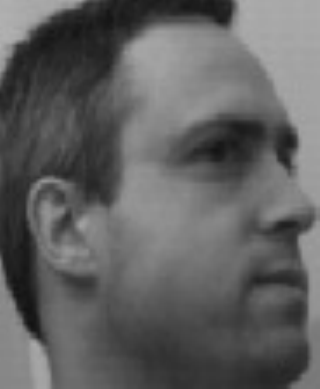}
	\includegraphics[width=1cm]{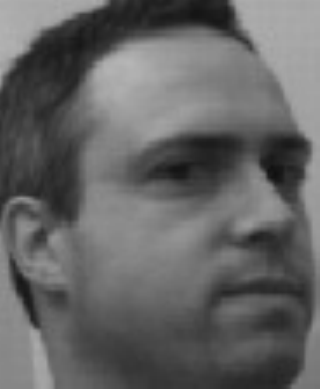}
	\includegraphics[width=1cm]{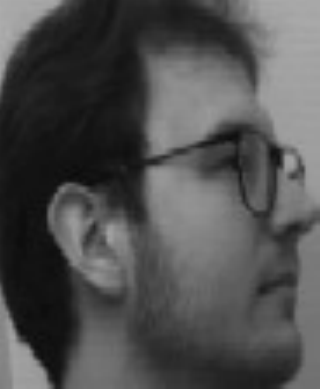}
	\includegraphics[width=1cm]{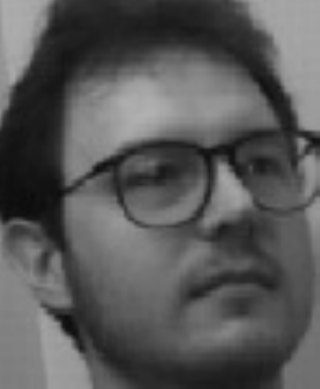}
	\caption{Some examples from the UMIST face dataset.}
	\label{Figure3}
\end{figure}

Based on the query man wearing glasses, we can label the dataset in the following way: man wearing glasses is regarded as completely relevant, which is labeled as 2 in our experiment; man not wearing glasses or woman wearing glasses is regarded as partially relevant, which is labeled as 1; woman not wearing glasses is regarded as irrelevant, which is labeled as 0.
Thus, there are 225, 239 and 111 images in the three categories, respectively.
The dimension of original data is 10304.

In this experiment, for LDMLR, we set iteration number $T_{\hbox{max}}=30$ and the tradeoff parameter $\mu=10^3$ according to \citet{Li2015Ordinal}. For our method, we set parameters $\mu=10^{-10}$, $\gamma=10^{-9}$, $\rho=0.5$, $\sigma=0.05$, the maximum number for line search loop is $r=20$. To get an EDM $D$ in $(\ref{D})$, we set parameter $\beta=1$ ($\mu_0=0$ in this situation). To apply our algorithm, we first use PCA to reduce dimension as done in \citet{Li2015Ordinal}. When using PCA, we center the data but don't scale the data. The final dimension is $150$, i.e., $d=150$.

{\bf Role of the Embedding Dimension $s$ and Distance Metric}
To see the role of the Embedding dimension $s$ and distance metric, we do the following test. We randomly select 10 images from each label for training and use the rest for testing. The images in {the} training set are grouped as follows. The training data $\mathbf{x}_1,\cdots,\mathbf{x}_{10}$ are of label $1$, $\mathbf{x}_{11},\cdots,\mathbf{x}_{20}$ are of label $2$, and $\mathbf{x}_{21},\cdots,\mathbf{x}_{30}$ are of label $3$. Then there are $n=30$ training data in total. We fix the number of target neighbors as $K=5$.

\begin{table}[!ht]
	\centering
	\caption{Results of cMDS-DML on the UMIST dataset, with different values of dimension $s$.}
	\label{table0}
	\begin{tabular}{c|ccccc}
		\hline
		 ${s}$                  & 2 & 3 & 5 & 8 & 10   \\		
		\hline
		MAE                 & 0.3539 & 0.3463 & 0.3498 & 0.3684 & 0.3798 \\
		STD                　& 0.0812 & 0.0671 & 0.0640 & 0.0762 & 0.0830 \\
		t(s)                & 2.27   &  3.10  & 4.58   & 6.60   & 10.08   \\
		\hline
	\end{tabular}
\end{table}

To choose a proper embedding dimension, we tried several values for $s$, i.e., $s = 2, 3, 5, 8, 10$. The preliminary results are reported in Table $\ref{table0}$. Since $n =30$ is not so big, we directly apply cMDS to $D$ in S2 of Algorithm \ref{alg}. The observation is that $s=3$ and $s=5$ are the best in terms of MSE. Taking visualization into account, we choose $s=3$ in our following test.

Then we compute the Euclidean distance between {the} training data $\mathbf{x}_i,\ \mathbf{x}_j,\ i,j=1,\cdots,30$. Fig. $\ref{Figure2}$ shows $\|\mathbf{x}_i-\mathbf{x}_1\|$, the Euclidean distance between $\mathbf{x}_i$ and the first data $\mathbf{x}_1$, $i = 1, \cdots, n$.
It is observed that the distance between $\mathbf{x}_1$ and $\mathbf{x}_{12}$ is less than the distance between $\mathbf{x}_1$ and $\mathbf{x}_8$. Moreover, the distance between $\mathbf{x}_1$ and $\mathbf{x}_{16}$ is bigger than the distance between $\mathbf{x}_1$ and $\mathbf{x}_{22}$. It implies that the Euclidean distances between the original images  can not be used for ranking.
With embedding dimension $s=3$, we apply our method to learn $L$.  After learning $L$, the embedding points of the training data in the three dimensional space can be found, i.e., $L\mathbf{x}_i$, $i=1,\cdots,30$. Fig. $\ref{Figure3D}$ plots the embedding points. As we can see, points highly cluster together with the same label. However, the distances between points with different labels can not be clearly seen from Fig. $\ref{Figure3D}$. We use the learned $L$ to measure the distances between the training data. Fig. $\ref{Figure1}$ illustrates the distances between $\mathbf{x}_i$ and $\mathbf{x}_1$ under $L$, i.e., $\|L\mathbf{x}_i-L\mathbf{x}_1\|$, $i = 1,\cdots, n$.
Comparing Fig. $\ref{Figure1}$ with Fig. $\ref{Figure2}$, we can see that the data is much better layered with the $L$ distance than with the Euclidean distance. Hence the proposed model does preserve the ordinal relationship.
\begin{figure}[!ht]\centering
    \includegraphics[width=1\textwidth]{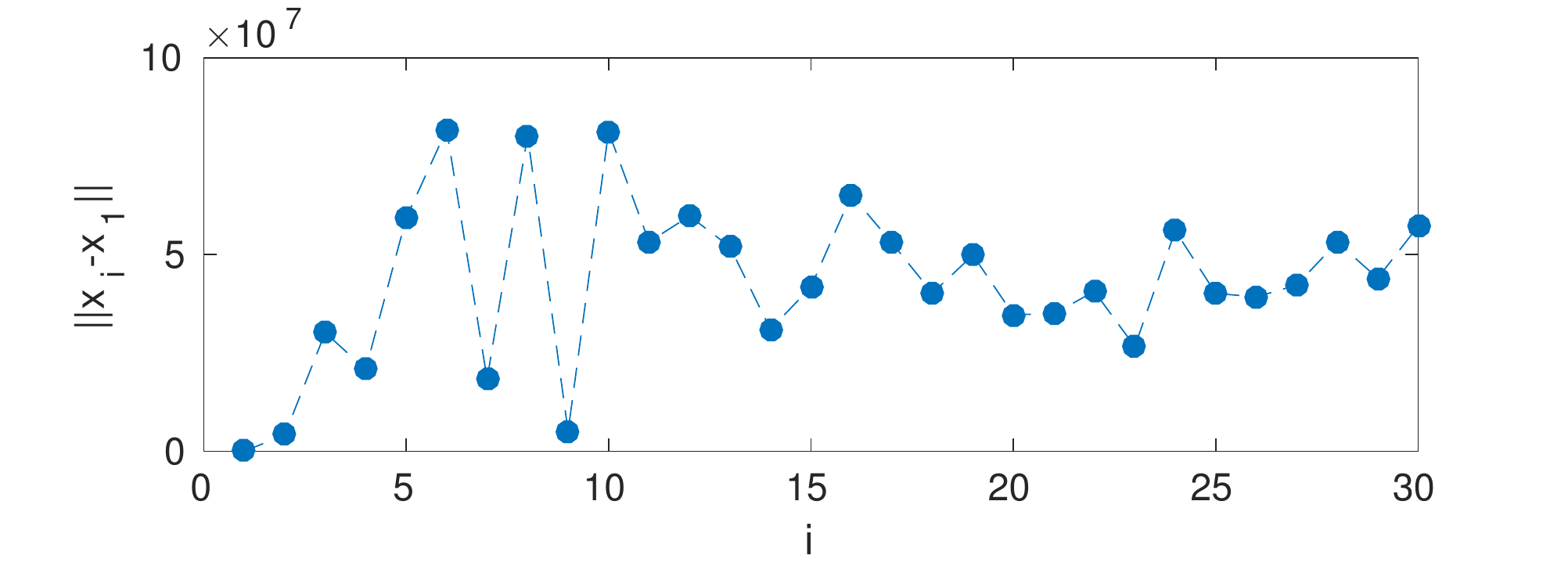}
	\caption{The Euclidean distance between $\mathbf{x}_i$ and $\mathbf{x}_1$, i.e., $\|x_i-x_1\|$.}
	\label{Figure2}
\end{figure}

\begin{figure}[!ht]\centering
	\includegraphics[width=1\textwidth]{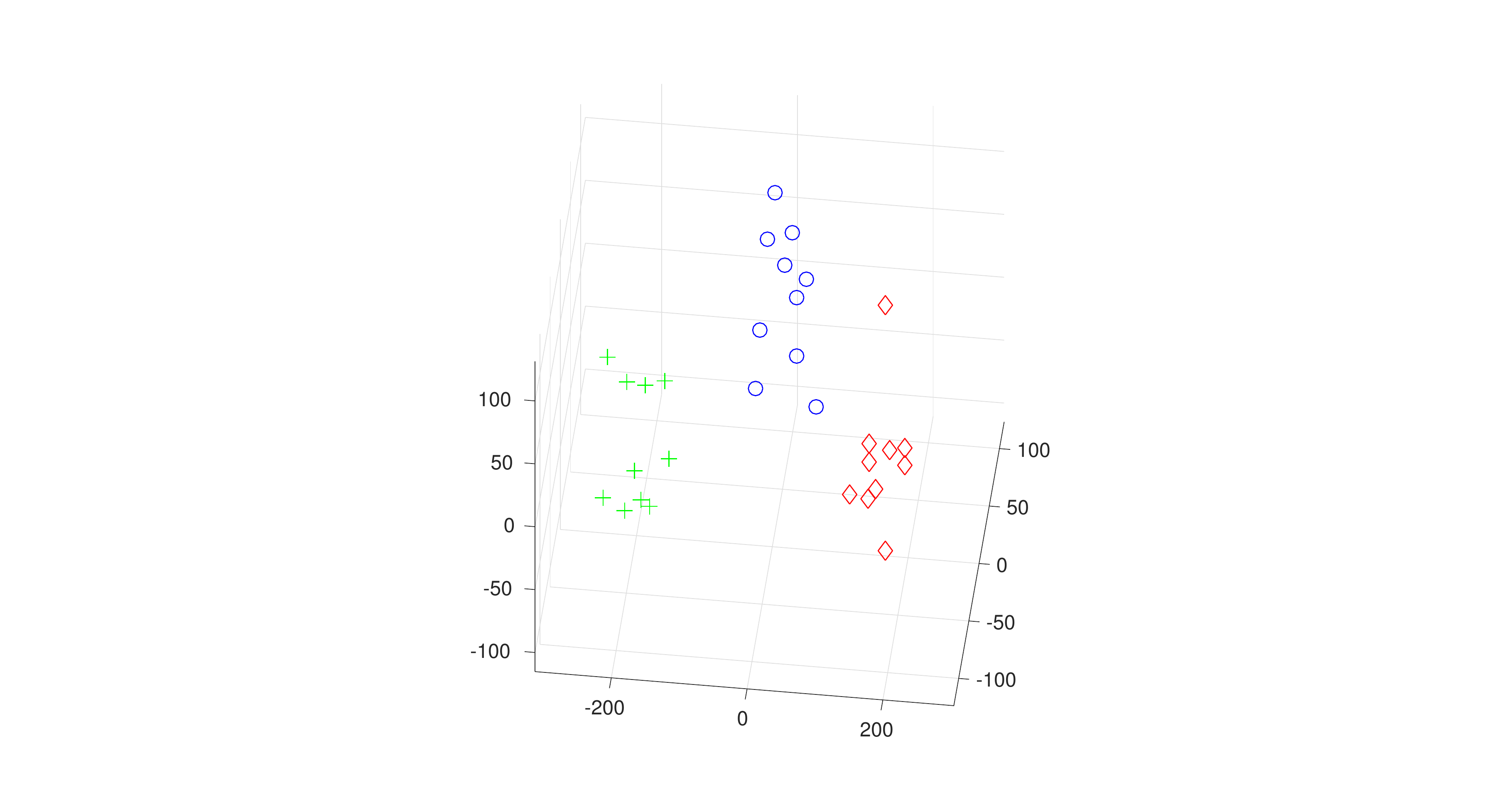}
	\caption{The embedding data points of {the} training data points in the three dimensional space, i.e., $L\mathbf{x}_i$.}
	\label{Figure3D}
\end{figure}

\begin{figure}[!ht]\centering
    \includegraphics[width=1\textwidth]{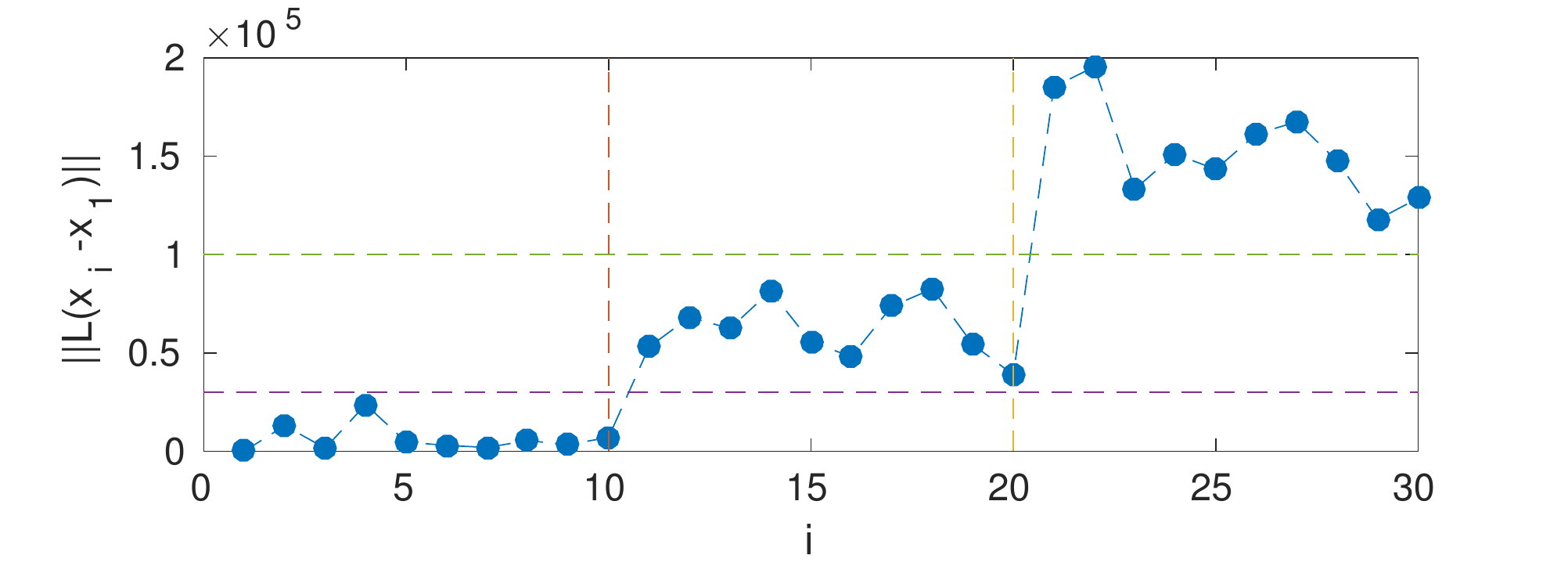}
	\caption{The distance between $\mathbf{x}_i$ and $\mathbf{x}_1$ under $L$, i.e., $\|L(\mathbf{x}_i-\mathbf{x}_1)\|$.}
	\label{Figure1}
\end{figure}
{\bf Comparison with LDMLR}
Now we compare with LDMLR in \citet{Li2015Ordinal}.
First, we randomly select 10 images from each distinct label as {the} training data and use the rest for testing. Different values of $K$ are chosen to investigate the performance.
Table $\ref{table3}$ gives the results including MAE, STD (standard deviation), and {CPU time} in seconds. We can see
in all cases, cMDS-DML uses much less time than LDMLR, which is not surprising since our method has lower computational complexity. In terms of MAE, cMDS-DML also outperforms LDMLR.
\begin{table}[!ht]
	\centering
	\caption{Results of cMDS-DML and LDMLR on the UMIST dataset, with different values of target neighbors $K$ and fixed $n=30$.}
	\label{table3}
	\begin{tabular}{cccc}
		\hline
		${K}$ &        & cMDS-DML & LDMLR  \\
		\hline
		\multirow{3}{*}{4} & MAE    & 0.3488     & 0.4291 \\
		& STD    &$\pm0.0684$ & $\pm0.0760 $ \\
		& t(s)&  2.88     &  10.26 \\
		\hline
		\multirow{3}{*}{5} & MAE    & 0.3463      & 0.4676 \\
		&STD     & $\pm0.0671$ & $\pm0.0735$   \\
		& t(s)&  3.10      &10.39   \\
		\hline
		\multirow{3}{*}{6}  &MAE     & 0.3521     & 0.4782 \\
		& STD    &$\pm0.0689$ & $\pm0.0724$  \\
		& t(s)&   3.32   &  11.92\\
		\hline
	\end{tabular}
\end{table}

Next, to evaluate the influence of dimension on the performance of our method, we increase dimension $d$ while fixing the size of {the} training set $n=30$ and the number of target neighbors $K=5$. Table $\ref{table4}$ lists the ranking results. It can be seen that as $d$ increases, the resulting MAE of both algorithms is not sensitive to $d$.
As for computing time, as $d$ increases, LDMLR obviously costs more time while CUP time for our method is fairly stable.
It's reasonable since the computational complexity of our method is {proportional} to $d^2$ while that of LDMLR is $d^3$.

\begin{table}[!ht]
	\centering
	\caption{Results of cMDS-DML and LDMLR on the UMIST dataset, with different values of dimension $d$.}
	\label{table4}
	\begin{tabular}{cccc}
		\hline
		${d}$ &          & cMDS-DML & LDMLR  \\
		\hline
		\multirow{3}{*}{150} & MAE    & 0.3463      & 0.4676 \\
		&STD     & $\pm0.0671$ & $\pm0.0735$   \\
		& t(s)&  3.10      &10.39   \\
		\hline
		\multirow{3}{*}{200} & MAE    & 0.3518      & 0.4695 \\
		&STD     & $\pm0.0668$ & $\pm0.0730$   \\
		& t(s)&  3.18     &17.21   \\
		\hline
		\multirow{3}{*}{250}  &MAE     & 0.3545     & 0.4708 \\
		& STD    &$\pm0.0674$ & $\pm0.0742$  \\
		& t(s)&   4.15    &  23.90\\
		\hline
	\end{tabular}
\end{table}

Finally, we increase the size of {the} training set $n$ with fixed  dimension $d=150$ and the number of target neighbors $K=5$. We randomly select $n/3$ images from each distinct label for training and report results in Table $\ref{table5}$.
As $n$ increases, the performance of both methods becomes better, which is reasonable. cMDS-DML achieves higher ranking performance than LDMLR. In particular, cMDS-DML achieves $25.94\%$, $36.90\%$, $39.36\%$ improvement in MAE ($|$MAE(LDMLR)-MAE(cMDS-DML)$|$/MAE(LDMLR)) over LDMLR, respectively. Moreover, cMDS-DML is also  faster than LDMLR.
\begin{table}[!ht]
	\centering
	\caption{Results of cMDS-DML and LDMLR on the UMIST dataset, with different sizes of the training set ${n}$.}
	\label{table5}
	\begin{tabular}{cccc}
		\hline
		${n}$ &          & cMDS-DML & LDMLR  \\
		\hline
		\multirow{3}{*}{30} & MAE    & 0.3463      & 0.4676 \\
		&STD     & $\pm0.0671$ & $\pm0.0735$   \\
		& t(s)&  3.10      &10.39   \\
		\hline
		\multirow{3}{*}{60} & MAE    & 0.1958      & 0.3103 \\
		&STD     & $\pm0.0478$ & $\pm0.0465$   \\
		& t(s)&  4.94      &39.13   \\
		\hline
		\multirow{3}{*}{90}  &MAE     & 0.1373     & 0.2264 \\
		& STD    &$\pm0.0329$ & $\pm0.0319$  \\
		& t(s)&   5.91    &  79.92\\
		\hline
	\end{tabular}
\end{table}

\subsection{Experiments on the FG-NET dataset}
In this experiment, we test our algorithm on
the FG-NET dataset which {is} labeled by age.
The FG-NET dataset contains 1002 face images.
There are 82 subjects in total with the age ranges from 1 to 69. Fig. $\ref{Figure4}$ shows some examples from the FG-NET dataset.
To get better performance of LDMLR, we set the iteration number $T_{\hbox{max}}=50$ and the tradeoff parameter $\mu=10^3$. For our method, $\mu=10^{-10}$, $\gamma=10$, $\rho=0.5$, $\sigma=0.05$, the embedding dimension $s=3$ and the maximum number for line search loop is $r=20$.
\begin{figure}[!ht]\centering
	\includegraphics[width=1cm,height=1cm]{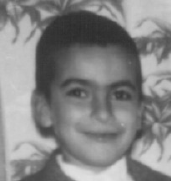}
	\includegraphics[width=1cm,height=1cm]{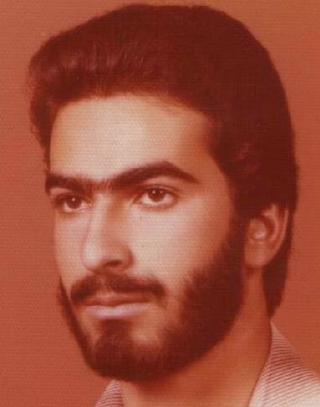}
	\includegraphics[width=1cm,height=1cm]{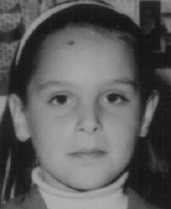}
	\includegraphics[width=1cm,height=1cm]{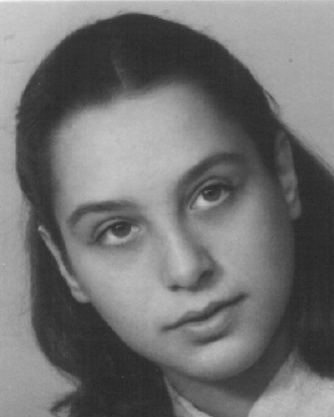}
	\includegraphics[width=1cm,height=1cm]{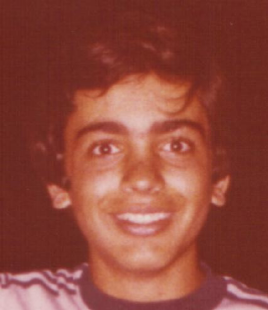}\\
	\includegraphics[width=1cm,height=1cm]{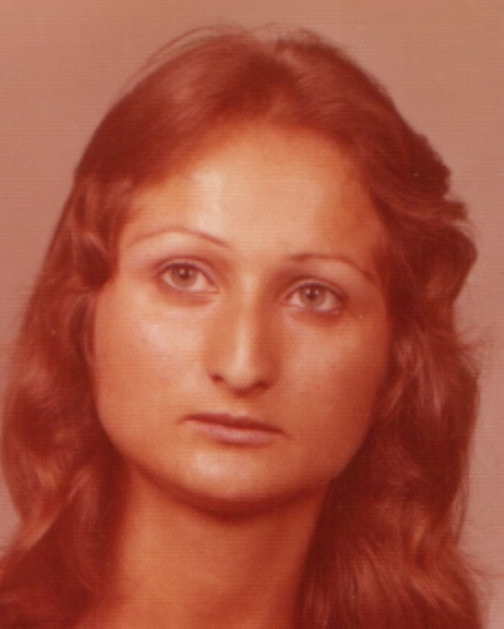}
	\includegraphics[width=1cm,height=1cm]{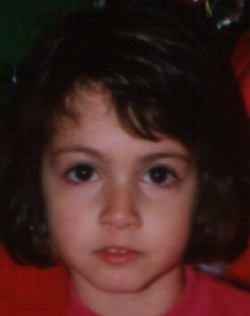}
	\includegraphics[width=1cm,height=1cm]{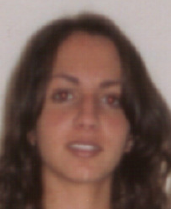}
	\includegraphics[width=1cm,height=1cm]{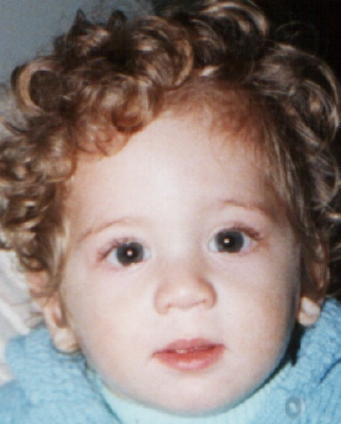}
	\includegraphics[width=1cm,height=1cm]{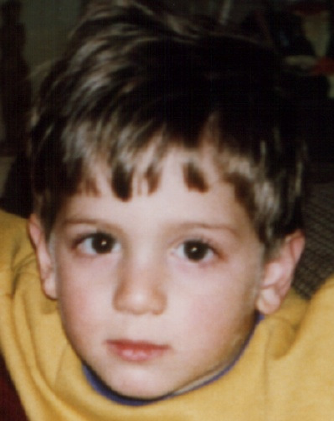}
	\caption{Some examples from the FG-NET dataset.}
	\label{Figure4}
\end{figure}

\begin{figure}[!ht]\centering
	\includegraphics[width=1\textwidth]{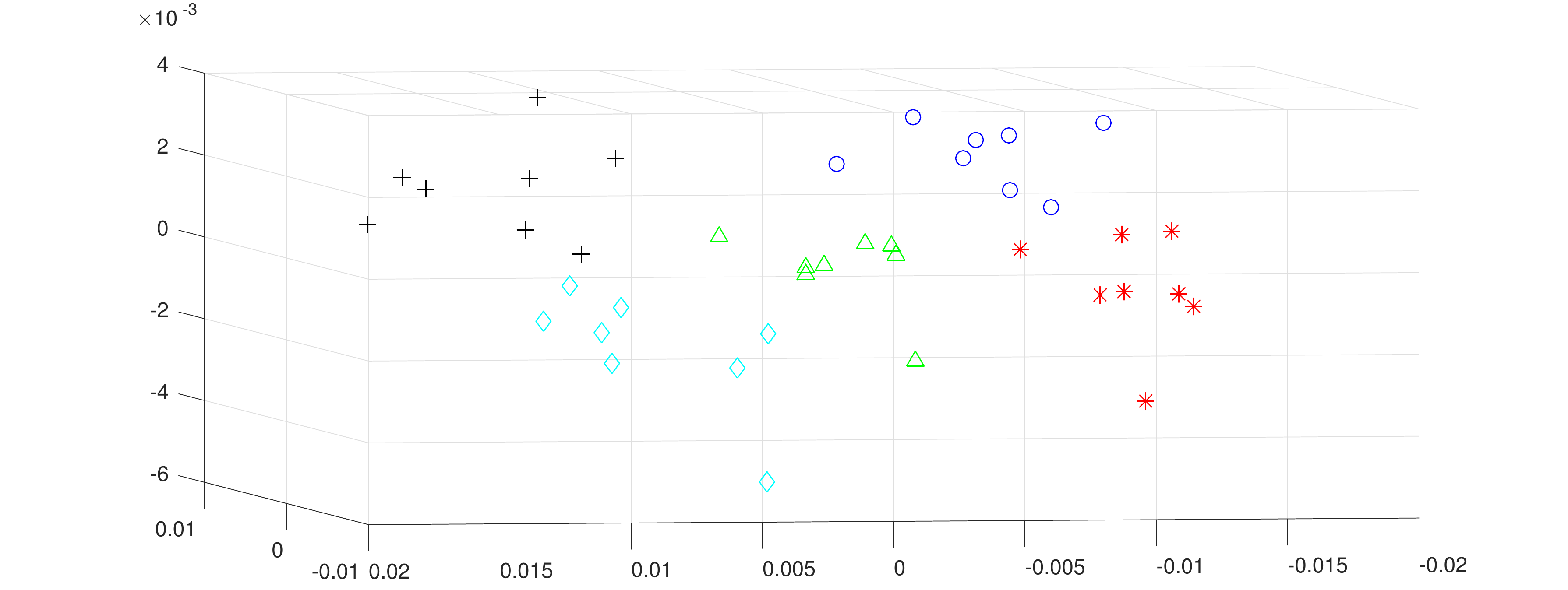}
	\caption{The embedding data points of {the} training data points in the three {dimensional} space, i.e., $L\mathbf{x}_i$.}
	\label{Figure3D2}
\end{figure}
We pick up subjects with age 1, 5, 9, 15, 19 and relabel them as 1, 2, 3, 4, 5.
There are 27, 40, 25, 30, 23 images in the five categories, respectively. The original dimension of images is 136.
As in subsection $5.1$, we preprocess the data by PCA to reduce dimension to $80$.
We randomly select 8 images from each distinct label for training and set $K=5$. Fig. $\ref{Figure3D2}$ plots the embedding data points of {the} training data points in the three {dimensional} space, i.e., $L\mathbf{x}_i$, $i=1,\cdots,40$. As we can see, points almost cluster together with the same label.

Next, we randomly select 10 images from each distinct label for training and use the rest for testing. That is, the size of {the} training set is $n=50$. We also set $\beta=1$ in $(\ref{D})$. We set different values for target neighbors to investigate the performance. Table $\ref{table6}$ lists the experimental results.
In the three cases, cMDS-DML achieves $48.55\%$, $47.27\%$, $46.57\%$ improvement over LDMLR, respectively.

\begin{table}[!ht]
	\centering
	\caption{Results of cMDS-DML and LDMLR on the FG-NET dataset, with different values of target neighbors ${K}$.}
	\label{table6}
	\begin{tabular}{cccc}
		\hline
		${K}$ &          & cMDS-DML & LDMLR  \\
		\hline
		\multirow{3}{*}{4} & MAE    & 0.7295      & 1.4179 \\
		&STD     & $\pm0.0694$ & $\pm0.1228$   \\
		& t(s)&  2.70      &21.59   \\
		\hline
		\multirow{3}{*}{5} & MAE    & 0.7109      & 1.3482 \\
		&STD     & $\pm0.0657$ & $\pm0.0942$   \\
		& t(s)&  3.11      &18.37   \\
		\hline
		\multirow{3}{*}{6}  &MAE     & 0.7095     & 1.3278 \\
		& STD    &$\pm0.0682$ & $\pm0.1141$  \\
		& t(s)&   3.86    &  25.43\\
		\hline
	\end{tabular}
\end{table}

\begin{table}[!ht]
	\centering
	\caption{Results of cMDS-DML and LDMLR on the FG-NET dataset, with different sizes of the training set ${n}$.}
	\label{table7}
	\begin{tabular}{cccc}
		\hline
		${n}$ &          & cMDS-DML & LDMLR  \\
		\hline
		\multirow{3}{*}{40} & MAE    & 0.7613      & 1.3634 \\
		&STD     & $\pm0.0845$ & $\pm0.1144$   \\
		& t(s)&  2.34      &11.13   \\
		\hline
		\multirow{3}{*}{50} & MAE    & 0.7377      & 1.3482 \\
		&STD     & $\pm0.0657$ & $\pm0.0942$   \\
		& t(s)&  4.56      &18.37   \\
		\hline
		\multirow{3}{*}{75}  &MAE     & 0.7243     & 1.2551 \\
		& STD    &$\pm0.0809$ & $\pm0.1023$  \\
		& t(s)&   29.97    &  40.04\\
		\hline
	\end{tabular}
\end{table}
Finally, we fix the value of target neighbors $K=5$. We randomly select $n/5$ images from each distinct label for training. The  size of {the} training set is chosen as $n=40, \ 50, \ 75$. See Table $\ref{table7}$ for the results, which again verify the efficiency of the proposed model.

Overall speaking,  our numerical results show that cMDS-DML outperforms LDMLR significantly both in ranking performance and CPU time.
\section{Conclusions}\label{sec-conclusion}
In this paper, we proposed a so-called cMDS-DML approach for image ranking, which unifies the idea of classical multidimensional scaling and distance metric learning. The algorithm enjoys low computational complexity, compared with LDMLR in \citet{Li2015Ordinal}. Numerical results verified the efficiency of the new approach and the improvement over LDMLR.

%\bibliographystyle{apalike}
%\bibliography{documentnew3}
%\include{documentnew33}

\bibliographystyle{plain}
\bibliography{paper}
\end{document}